\documentclass[11pt]{article}
\usepackage{parskip,fullpage}
\usepackage{amsmath}
\usepackage{xcolor}

\usepackage{natbib}
\usepackage{url,hyperref}
\hypersetup{
    colorlinks=true,
    linkcolor=teal,
    filecolor=teal,     
    citecolor = teal,
    urlcolor=teal,
}

\usepackage{amsmath}
\usepackage{amsthm}
\usepackage{amsfonts}

\usepackage{graphicx}
\usepackage{epsfig}
\usepackage{subfigure}

\theoremstyle{plain}
\newtheorem{theorem}	 			{Theorem}
\newtheorem{lemma}		[theorem]	{Lemma}

\newtheorem{corollary}		[theorem]	{Corollary} 
 
\theoremstyle{definition}
\newtheorem{definition}[theorem]{Definition} 
\theoremstyle{remark} 

\usepackage{xcolor}

\title{Noise in Classification\thanks{
Chapter 16~\citep{balcan2020chapter} of the book Beyond the Worst-Case Analysis of Algorithms~\citep{roughgarden2020bwca}}}
\author{Maria-Florina Balcan\\ Carnegie Mellon University\\\texttt{ninamf@cs.cmu.edu} \and Nika Haghtalab\\ Cornell University\\University of California, Berkeley\\\texttt{nika@berkeley.edu}}
\date{}

\begin{document}

\maketitle
\renewcommand{\vec}[1]{\mathbf{#1}}

\newcommand{\red}[1]{{\color{red}#1}}
\newcommand{\blue}[1]{{\color{blue}#1}}

\newcommand{\X}{\mathcal{X}}
\newcommand{\Y}{\mathcal{Y}}
\newcommand{\D}{\mathcal{D}}
\renewcommand{\P}{\mathcal{P}}
\newcommand{\F}{\mathcal{F}}
\renewcommand{\O}{\mathcal{O}}
\newcommand{\R}{\mathbb{R}}
\newcommand{\Z}{\mathbb{Z}}

\newcommand{\EE}{\operatorname*{\mathbb{E}}}
\newcommand{\coloneqq}{\operatorname*{:=}}
\newcommand{\argmin}{\operatorname*{argmin}}
\newcommand{\err}{\operatorname*{\mathrm{err}}}
\newcommand{\opt}{\mathrm{opt}}

\setcounter{page}{1}

\newcommand{\darrow}{{\downarrow}}

\newcommand{\sign}{\mathrm{sign}}
\newcommand{\vcd}{\mathrm{dim}}
\newcommand{\poly}{\mathrm{poly}}

\newcommand{\Card}[1]{\left|#1\right|}
\newcommand{\Parens}[1]{\left(#1\right)}
\newcommand{\Braces}[1]{\left\{#1\right\}}
\newcommand{\Bracks}[1]{\left[#1\right]}

\newcommand{\card}[1]{|#1|}
\newcommand{\parens}[1]{(#1)}
\newcommand{\braces}[1]{\{#1\}}
\newcommand{\bracks}[1]{[#1]}

\allowdisplaybreaks

\begin{abstract}
This chapter considers the computational and statistical aspects of learning linear thresholds in presence of noise. When there is no noise, several algorithms exist that  efficiently learn near-optimal linear thresholds using a small amount of data. However, even a small amount of adversarial noise makes this problem notoriously  hard in the worst-case. 
We discuss approaches for dealing with these negative results by exploiting natural assumptions on the  data-generating process.
\end{abstract}

\section{Introduction}

Machine learning studies automatic methods for making  accurate predictions and useful decisions based on previous observations and experience.
From the application point of view, machine learning has become a successful discipline for operating in complex domains such as natural language processing, speech recognition, and computer vision.
Moreover, the theoretical foundations of machine learning have led to the development of powerful and versatile techniques, which are routinely used in a wide range of commercial systems in today's world. However, a major challenge of increasing importance in the theory and practice of machine learning is to provide algorithms that are robust to adversarial noise.

In this chapter, we focus on \emph{classification} where the goal is to learn a classification rule from labeled examples only.
Consider for example the task of automatically classifying social media posts as either appropriate or inappropriate for publication.
To achieve this one can examine past social media posts and their features--- such as the author, bag of words, and hashtags---and whether they were appropriate for publication.
This data can be used to learn a classifier that best decides whether a newly posted article is appropriate for publication, e.g., by finding the best parameters for a linear classifier that makes such predictions.
While classification is one of the most commonly used paradigms of machine learning in practice, from the worst-case perspective, it is often computationally hard and requires an amount of information that is seldom available. 
One of the main challenges in classification is the presence of \emph{noise} in data sets.
For example, a labeler may be making mistakes when deciding whether a post is appropriate for publication. It is also possible that the correct decision is not a  linear separator or
even that there is no perfect classification rule.  The latter can happen when, for example, one appropriate post and one inappropriate post map to the same feature vector.
Indeed, when the noise in classification is adversarially designed, classification is believed to be hard.
On the other hand, there is an increasing need for learning algorithms that can withstand intentionally adversarial behavior in their environment, e.g., large fraction of inappropriate social media posts are made with the intention to pass through the deployed classifier.
Therefore, it is essential to provide a theoretical grounding for the performance of learning algorithms in presence of real-life adversaries.

In this chapter, we go beyond the worst-case in studying noise in classification with a focus on learning linear threshold classifiers.
Learning linear thresholds is a canonical problem in machine learning that serves as a backbone for several other learning problems, such as  support vector machines and neural networks. In the absence of noise, several computationally efficient algorithms exist that can learn highly accurate linear thresholds. However, introducing even a small amount of adversarial noise makes this problem notoriously intractable with a runtime that is prohibitively large in the number of features.
In this chapter, we present some of the recent progress on learning linear thresholds in presence of noise under assumptions on the data-generating process.
The first approach considers restrictions on the marginal distribution of instances, e.g., log-concave or Gaussian distributions. The second approach additionally considers how the true classification of instances match those expressed by the most accurate classifier one is content with, e.g., assuming that the Bayes optimal classifier is also a linear threshold. At a technical level, many of the results in this chapter 
contribute to and draw insights from high dimensional geometry to limit the effects of noise on learning algorithms.

We describe the formal setup in Section~\ref{sec:22:model}.
In Section~\ref{sec:22:worst}, we overview the classical worst-case and best-case results on computational and statistical aspects of classification.
In Section~\ref{sec:22:marginal}, we showcase general assumptions on the marginal distribution of instances that lead to improved computational performance.
In Section~\ref{sec:22:noise}, we obtain further computational and statistical improvements by investigating additional assumptions on the nature of the label noise.
We end this chapter by putting the work into a broader context.

\section{Model}
\label{sec:22:model}
We consider an instance space $\X$ and a set of labels $\Y = \{ -1, +1\}$.
A \emph{classifier} is a function $f: \X \rightarrow \Y$ that maps an instance $x \in \X$ to its classification $y$.
For example, 
$x$ can represent a social media post and $y$ can indicate whether the post was appropriate for publication.
We consider a set of classifiers $\F$.
We denote the VC dimension of $\F$, which measures the expressibility of $\F$, by $\vcd(\F)$.\footnote{VC dimension is the size of the largest $X\subseteq \X$ that can be labeled in all possible ways using functions in $\F$.}
We further consider a distribution $\D$ over $\X\times\Y$.
While $\D$ is assumed to be unknown, we assume access to a set $S$ of i.i.d. samples from $\D$.
For a classifier $f$, we represent its \emph{expected} and \emph{empirical error}, respectively, by $\err_\D(f) = \Pr_{(x,y)\sim \D}\left[y \neq f(x) \right]$ and $ \err_S(f) = \frac{1}{\card{S}} \sum_{(x,y)\in S} \mathbf{1}_{\left( y \neq f(x) \right) }.$
\emph{Classification} is the task of learning a classifier with near optimal error from a set of classifiers $\F$, i.e., finding $f$ such that $\err_\D(f) \leq \opt + \epsilon$, where $\opt = \min_{f^*\in \F} \err_\D(f^*)$. 
Classification is considered under a number of settings.
\emph{Agnostic learning}  refers to a setting where no additional assumptions are made regarding the set of classifiers $\F$ or the distribution of instances $\D$.
\emph{Realizable  learning} refers to a setting where there is  $f^* \in \F$ such that $\err_\D(f^*) = 0$, in which case, one is looking for a classifier $f$ with $\err_\D(f) \leq \epsilon$.

In parts of this chapter, we work with the class of linear threshold classifiers. That is, we assume that the input space is $\X = \R^d$ for some $d \in \mathbb{N}$ and refer to an instance by its $d$-dimensional vector representation $\vec x \in \R^d$. A \emph{homogeneous linear threshold} classifier, also called a \emph{halfspace through the origin}, is a function $h_{\vec w}(\vec x) = \sign(\vec w \cdot \vec x)$ for some unit vector $\vec w\in \R^d$. The VC dimension of the class of $d$-dimensional homogeneous linear thresholds is $\vcd(\F) = d$.

\section{The Best Case and The Worst Case}
\label{sec:22:worst}
In this section, we review the computational and statistical aspects of classification at the opposite ends of the difficulty spectrum---the realizable and agnostic settings.

\subsection{Sample Complexity}
\label{sec:22:sample-complexity}
It is well-known that to find a classifier of error $\epsilon$ in the realizable setting, all one needs to do is to take a set of $\tilde\Theta(\vcd(\F)/\epsilon)$ i.i.d. samples from $\D$ and choose $f\in \F$ that perfectly classifies all of these samples.\footnote{We use notation $\tilde\Theta$ to hide logarithmic dependence on $1/\epsilon$ and $1/\delta$.} Formally, for any $\epsilon, \delta \in (0, 1)$, there is
\[ m^{real}_{\epsilon, \delta} \in O\left(\frac{1}{\epsilon} \Parens{\vcd(\F) \ln\Big( \frac 1 \epsilon\Big) + \ln\Big(\frac 1\delta\Big)} \right)
\]
such that with probability $1-\delta$ over $S\sim \D^{m^{real}_{\epsilon, \delta}}$,  if $\err_S(f) = 0$ then $\err_\D(f)\leq \epsilon$.

In most applications of machine learning, however, either the perfect classifier is much more complex than those included in $\F$ or there is no way to perfectly classify instances.
This is where agnostic learning comes in. With no assumptions on the performance of classifiers one instead chooses $f\in \F$ with least empirical error. 
One way to make this work is to estimate the error of all classifiers within $\epsilon$ with high probability over the draw of the samples.
This is called \emph{uniform convergence} and requires $\Theta(\vcd(\F)/\epsilon^2)$ samples.
More formally, for any $\epsilon, \delta \in (0,1)$, there is
\[ m_{\epsilon, \delta} \in O\left(\frac{1}{\epsilon^2} \Parens{\vcd(\F) + \ln\Big(\frac 1\delta\Big)} \right)
\]
such that with probability $1-\delta$ over the sample set $S\sim \D^{m_{\epsilon, \delta}}$, for all $f\in \F$, $\Card{\err_\D(f) - \err_S(f)}\leq \epsilon$. These sample complexities are known to be nearly tight.
We refer the reader to \citet{AB99} for more details. \footnote{
Informally speaking, the VC dimension of a function class controls the number of samples needed for learning. This is because if the number of samples is much smaller than the VC dimension, it is possible to have two functions that perform identically  on the training set but there is a large gap between their performance on the true distributions. The surprising aspect of these results is that VC dimension also characterizes the number of samples sufficient for learning.
}

It is evident from these results that in the worst case agnostic learning requires significantly, i.e., about a factor of $1/\epsilon$, more data than realizable learning. Unfortunately, such large amount of data may not be available in many applications and domains, e.g., medical imaging. On the other hand, day-to-day applications of machine learning rarely resemble worst-case instances of agnostic learning. In Section~\ref{sec:22:noise}, we show how the sample complexity of agnostic leaning significantly improves when we make additional assumptions on the nature of the noise.

\subsection{Computational Complexity}
\label{sec:comp-complexity:worstcase}

Given a sufficiently large sample set, the computational complexity of classification is concerned with whether one can efficiently compute a classifier of good quality on the samples.
In the realizable setting, this involves computing a classifier $f$ that makes no mistakes on the sample set.
More generally, in the agnostic setting one needs to compute a classifier $f\in \F$ of (approximately) least error. 
This can be done in $\poly(\Card{\F})$ runtime.
However, in most cases $\F$ is infinitely large or, even in cases that it is finite, it is exponential in the natural representation of the problem, e.g., the set of all linear threshold functions, decision trees, boolean functions, etc.
In this section, we focus on a setting where $\F$ is the set of homogeneous linear thresholds,
which is one of the most popular classifiers studied in  machine learning.

Consider the realizable setting where  $f_{\vec w} \in \F$ exists that is consistent with the set  $S$ sampled from $\D$, i.e.,   $y = \sign(\vec w\cdot \vec x)$
for all $(\vec x, y)\in S$. Then  such a vector $\vec w = \frac{\vec v}{\| \vec v\|_2}$ can be computed in time $\poly(d, |S|)$ by finding a solution $\vec v$ to the following linear program with a dummy objective 
\begin{equation}
\begin{array}{ll@{}ll}
\text{minimize}_{\vec v\in \R^d}  & 1 &\\
\text{subject to} & y(\vec v \cdot \vec x) \geq 1,    \qquad \forall (\vec x, y) \in S.
\end{array}
\label{eq:LP}
\end{equation}

Can one use this linear program in the agnostic case? The answer  depends on how much noise, measured by the error of the optimal classifier, exists in the data set. After all, if the noise is so small that it does not appear in the sample set, then one can continue to use the above linear program.
To see this more formally, let $\O_\F$ be an \emph{oracle} for the realizable setting that takes sample set $S$ and returns a classifier $f\in \F$ that is perfect on $S$ if one exists. Note that the above linear program achieves this by returning $\vec w$ that satisfies the constraints or certifying that no such vector exists. In the following algorithm, we apply this oracle that is designed for the realizable setting to the agnostic learning problem  where the noise level is very small. Interestingly, these guarantees go beyond learning linear thresholds and apply to any learning problem that is efficiently solvable in the realizable setting.

\begin{figure}[h]
\hrule\medskip
\textbf{Input}: Sampling access to $\D$, set of classifiers $\F$, oracle $\O_\F$, $\epsilon$, and $\delta$.
\begin{enumerate}

\item Let $m = m^{real}_{\frac{\epsilon}{4}, 0.5}$ and $r = m^2 \ln(2/\delta)$.

\item For $i = 1, \dots, r$, take a sample set $S_i$ of $m$ i.i.d. samples from $\D$. Let $f_i = \O_\F(S_i)$ or $f_i = \text{``none''}$ if there is no perfect classifier for $S_i$. \label{step:call-oracle}

\item Take a fresh sample set $S$ of $m' = \frac{1}{\epsilon} \ln(r/ \delta)$ i.i.d. samples from $\D$.

\item Return $f_i$ with minimum $\err_S(f_i)$.

\end{enumerate}
\caption{Efficient Agnostic Learning for Small Noise.}\label{fig:alg:KL88-agnostic-reduction}
\medskip\hrule
\end{figure}

\begin{theorem}[\cite{KL88}] \label{thm:kearns-li}
Consider an agnostic learning problem with distribution $\D$ and set of classifiers $\F$ such that $\min_{f\in \F} \err_\D(f) \leq c\epsilon / \vcd(\F)$ for a sufficiently small constant $c$. 
Algorithm~\ref{fig:alg:KL88-agnostic-reduction} makes $\poly\Parens{\vcd(\F), \frac 1\epsilon, \ln(\frac 1\delta)}$ calls to the realizable oracle $\O_\F$, and with probability $1-\delta$, returns a classifier $f$ with $\err_\D(f) \leq \epsilon$.
\end{theorem}
\begin{proof}[Proof Sketch]
It is not hard to see that Step~\ref{step:call-oracle} of the algorithm returns at least one classifier $f_i \in \F$, which perfectly classifies $S_i$ of size $m = \tilde\Theta(\vcd(\F)/\epsilon)$ with high probability, and therefore has error at most $\epsilon/4$ on $\D$.
Since $m = \Theta( \epsilon^{-1}\vcd(\F)\ln(1/\epsilon))$, for a fixed $i$, the probability that $S_i$ is perfectly labeled by the optimal classifier is  
$\Parens{1- c\epsilon/\vcd(\F)}^m \geq \frac {1}{m^2}.$ Repeating this $r = m^2 \ln(2/\delta)$ times, with probability at least $1 - \frac \delta 2$ at least one sample set is perfectly labeled by the optimal classifier

We use the Chernoff bound to estimate the error within a multiplicative factor of $2$.\footnote{
Here, a multiplicative approximation rather than an additive one needs  fewer samples than is discussed in Section~\ref{sec:22:sample-complexity}.}
With probability $1-\delta$ over the choice of $m'$ samples $S$, any such classifier $f_i$ with $\err_\D(f_i) \leq \epsilon /4$ has  $\err_S(f_i) \leq \epsilon /2$  and any such classifier with $\err_\D(f_i)> \epsilon$ has $\err_S(f_i) >  \epsilon/ 2$. Therefore, Algorithm~\ref{fig:alg:KL88-agnostic-reduction} returns a classifier of error $\epsilon$.
\end{proof}

Theorem~\ref{thm:kearns-li} shows how to efficiently learn a $d$-dimensional linear threshold in the agnostic setting when the noise level is $O\Parens{\frac{\epsilon}{d}}$.
At its heart, Theorem~\ref{thm:kearns-li} relies on the fact that when the noise is small,
Linear Program~\eqref{eq:LP} is still feasible with a reasonable probability.
On the other hand, significant noise in agnostic learning leads to unsatisfiable constraints in Linear Program~\eqref{eq:LP}.
Indeed, in a system of linear equations where one can satisfy $(1-\epsilon)$ fraction of the equations, it is NP-hard to find a solution that satisfies $\Theta(1)$ fraction of them. \citet{GR09} use this to show that even if there is a near-perfect linear threshold $f^*\in \F$ with $\err_\D(f^*) \leq \epsilon$, finding any classifier in $\F$ with error $\leq \frac 12 - \Theta(1)$ is NP-Hard.
Agnostic learning of linear thresholds is hard even if the algorithm is allowed to return a classifier $f\notin \F$.
This setting is called  \emph{improper} learning and is generally simpler than the problem of learning a classifier from $\F$.
But, even in  improper learning, when the optimal linear threshold  has a small constant error $\mathrm{opt}$, it is hard to learn a classifier of error $O(\mathrm{opt})$~\citep{daniely2015hardness}, assuming that refuting random \emph{Constraint Satisfaction Problems} is hard under a certain regime.

These hardness results demonstrate a gap between what is efficiently learnable in the agnostic setting and the realizable setting, even if one has access to unlimited data.
In Sections~\ref{sec:22:marginal} and \ref{sec:22:noise}, we circumvent these hardness results by simple assumptions on the shape of the marginal distribution or the nature of the noise.

\section{Benefits of Assumptions on the Marginal Distribution}
\label{sec:22:marginal}

In this section, we show how additional assumptions on the marginal distribution of $\D$ on instance space $\X$ improve the computational aspects of classification. A commonly used class of distributions in the theory and practice of machine learning is the class of \emph{log-concave} distributions, which includes the Gaussian distribution and uniform distribution over convex sets.
Formally,
\begin{definition}
A distribution $\P$ with density $p$ is log-concave if $\log(p(\cdot))$ is concave. It is isotropic if its mean is at the origin and has a unit co-variance matrix.
\end{definition}
Throughout this section, we assume that the marginal distribution of $\D$ is log-concave and isotropic.
Apart from this
we make no further assumptions and allow for arbitrary label noise.
Let us first state several useful properties of isotropic log-concave distributions. We refer the interested reader to \citep{LV07,ABL17,BL13} for the proof of these properties.

\begin{theorem} \label{thm:log:prop}
Let $\P$ be an isotropic log-concave distribution over $\X = \R^d$.
\begin{enumerate}
\item All marginals of $\P$ are also isotropic log-concave distributions.
\label{eq:log:marginal}

\item For any $r$, $\Pr[\| \vec x\| \geq r \sqrt{d}] \leq \exp(-r+1)$.
\label{eq:log:tail}

\item There are constants $\overline{C}_1$ and $\underline{C}_1$, such that for any two unit vectors $\vec w$ and $\vec w'$, $
\underline{C}_1 \theta(\vec w, \vec w') \leq \Pr_{\vec x\sim \P}[\sign(\vec w\cdot \vec x) \neq  \sign(\vec w'\cdot \vec x)]  \leq \overline{C}_1 \theta(\vec w, \vec w'),$ where $\theta(\vec w, \vec w')$ is the angle between vectors $\vec w$ and $\vec w'$.
\label{eq:log:angle}

\item There are constants $\overline{C}_2$ and $\underline{C}_2$, such that for any unit vector $\vec w$ and $\gamma$, 
$
\underline{C}_2 \gamma \leq \Pr_{\vec x \sim \P}\left[ \Card{\vec x \cdot \vec w} \leq \gamma \right] \leq \overline{C}_2 \gamma.$
\label{eq:log:band}

\item For any constant $C_3$ there is a constant $C'_3$ such that for any two unit vectors $\vec w$ and $\vec w'$ such that $\theta(\vec w, \vec w')\leq \alpha < \pi/2$, we have
$
\Pr_{\vec x \sim \P} [\Card{\vec x \cdot \vec w} \geq  C'_3 \alpha
 \text{ and } \sign(\vec w\cdot \vec x) \neq \sign (\vec w' \cdot \vec x) ] \leq \alpha C_3.
$
\label{eq:log:wedge}
 
\end{enumerate}
\end{theorem}

Part~\ref{eq:log:marginal} of Theorem~\ref{thm:log:prop} is useful in establishing the other properties of log-concave distributions.
For example,  projection of $\vec x$ on any orthonormal subspace is equivalent to the marginal distribution over the coordinates of the new subspace  and thus forms an isotropic log-concave distribution.
This allows one to prove the rest of Theorem~\ref{thm:log:prop} by analyzing the projections of $\vec x$ on the relevant unit vectors $\vec w$ and $\vec w'$.
Part~\ref{eq:log:marginal} and the exponential tail property of log-concave distributions (as expressed in Part~\ref{eq:log:tail}) are used in Section~\ref{sec:22:poly-reg} to show that linear thresholds can be approximated using low degree polynomials over log-concave distributions.

Part~\ref{eq:log:angle} allows one to bound the error of a candidate classifier in terms of its angle to the optimal classifier.
In addition, the exponential tail of log-concave distributions implies that a large fraction of the distribution is in a band  around the decision boundary of a classifier  (Part~\ref{eq:log:band}). 
Furthermore, the exponential tail property --- when applied to regions that are progressively farther from the origin and are captured within the disagreement region ---  implies that  only a small part of the disagreement between a candidate classifier and the optimal classifier falls outside of this band (Part~\ref{eq:log:wedge}).
Sections~\ref{sec:ABL} and \ref{sec:22:massart} use these properties to localize the learning problem near a decision boundary of a candidate classifier and achieve strong computational results for learning linear thresholds.

\subsection{Computational Improvements via Polynomial Regression}
\label{sec:22:poly-reg}

One of the reasons behind the computational hardness of agnostic learning is that it involves a non-convex and non-smooth function, $\sign(\cdot)$. 
Furthermore, $\sign(\cdot)$ cannot be approximated \emph{uniformly over all $\vec x$} by low degree polynomials or other convex and smooth functions that can be efficiently optimized. 
However, one only needs to approximate $\sign(\cdot)$ in \emph{expectation} over the data-generating distribution.
This is especially useful when the marginal distribution has exponential tail, e.g., log-concave distributions, because it allows one to focus on approximating $\sign(\cdot)$ close to its decision boundary, at the expense of poorer approximations far from the boundary, but without reducing the (expected) approximation factor overall.

This is the idea behind the work of \citet{KKMS08} that showed that if $\D$ has a log-concave marginal, then one can learn a classifier of error $\opt + \epsilon$ in time $\poly\left(d^{\kappa(1/\epsilon)}\right)$, for a fixed function $\kappa$.
Importantly, this result establishes that a low-degree polynomial threshold approximates $\sign(\cdot)$ in expectation.
We state this claim below and refer the interested reader to \citet{KKMS08} for its proof.
\begin{theorem}[\citet{KKMS08}] \label{thm:approx-sign}
There is a function $\kappa$ such that for any log-concave (not necessarily isotropic) distribution $\P$ on $\R$, for any $\epsilon$ and $\theta$, there is a polynomial $q:\R\rightarrow\R$ of degree $\kappa(1/\epsilon)$
such that 
$\EE_{z\sim\P}\Bracks{| q(z) - \sign(z-\theta)|} \leq \epsilon.
$
\end{theorem}

This result suggests a learning algorithm that fits a polynomial, in $L_1$ distance, to the set of samples observed from $\D$.
This algorithm is formally presented in Algorithm~\ref{fig:alg:KKMS}.
At a high level, for a set $S$ of labeled instances, our goal is to compute a polynomial $p:\R^d \rightarrow \R$ of degree $\kappa(1/\epsilon)$ that minimizes $\EE_{(\vec x ,y)\sim S}\left[ |p(\vec x) - y|\right]$. This can be done in time $\poly(d^{\kappa(1/\epsilon)})$, by expanding each $d$-dimensional instance $\vec x$ to a $\poly(d^{\kappa(1/\epsilon})$-dimensional instance $\vec x'$ that includes all monomials with degree at most $\kappa(1/\epsilon)$ of $\vec x$. We can then perform an $L_1$ regression on $(\vec x',y)$s to find $p$ in time $\poly(d^{\kappa(1/\epsilon)})$ --- for example by using a linear program.
With $p$ in hand, we then choose a threshold $\theta$ so as to minimize the empirical error of $\sign\left(p(\vec x) - \theta\right)$. 
This allows us to use Theorem~\ref{thm:approx-sign} to prove the following result.

\begin{figure}[t]
\hrule\medskip
\textbf{Input}: Set $S$ of $\poly\Parens{\frac 1\epsilon d^{\kappa(1/\epsilon)}}$ samples from $\D$
\begin{enumerate}
\item Find polynomial $p$ of degree $\kappa(1/\epsilon)$ that minimizes $\EE_{(\vec x, y)\sim S}\Bracks{|p(\vec x) - y|}.$

\item Let $f(\vec x) = \sign\left(p(\vec x) - \theta\right)$ for $\theta\in[-1,1]$ that minimizes the empirical error on $S$.
\end{enumerate}
\caption{$L_1$ Polynomial Regresssion.}\label{fig:alg:KKMS}
\medskip\hrule
\end{figure}

\begin{theorem}[\citet{KKMS08}]
\label{thm:KKMS:exp}
The $L_1$ Polynomial Regression algorithm (Algorithm~\ref{fig:alg:KKMS})  achieves  $\err_\D(f) \leq \opt+ \epsilon/2$ in expectation over the choice of $S$.
\end{theorem}

\begin{proof}
Let $f(\vec x) = \sign( p(\vec x) - \theta)$ be the outcome of Algorithm~\ref{fig:alg:KKMS} on sample set $S$. It is not hard to see that the empirical error of $f$ is at most half of the $L_1$ error of $p$, i.e., 
$\err_S(f)\leq \frac 12 \EE_S\left[|y-p(\vec x)|\right]$,
where $\EE_S$ denotes expectation taken with respect to the empirical distribution $(x,y)\sim S$.
To see this, note that $f(\vec x)$ is wrong only if $\theta$ falls between $p(\vec x)$ and $y$.
If we were to pick $\theta$ uniformly at random from $[-1,1]$, then in expectation $f(\vec x)$ is wrong with probability $|p(\vec x) - y|/2$. But $\theta$ is specifically picked by the algorithm to minimize $\err_S(f)$,  therefore, it beats the expectation and also achieves 
$\err_S(f) \leq \frac 12 \EE_S\left[ |y - p(\vec x)| \right].
$

Next, we use Theorem~\ref{thm:approx-sign} to show that there is a polynomial $p^*$ that approximates the optimal classifier in expectation over a log-concave distribution.
Let $h^* = \sign(\vec w^*\cdot \vec x)$ be the optimal linear threshold for distribution $\D$. Note that $\vec w^*\cdot \vec x$ is a one-dimensional isotropic log-concave  distribution (By Theorem~\ref{thm:log:prop} part~\ref{eq:log:marginal}).
Let $p^*(\vec x) = q(\vec x\cdot \vec w^*)$ be a polynomial of degree $\kappa(2/\epsilon)$ that approximates $h^*$ according to Theorem~\ref{thm:approx-sign}, i.e., $\EE_\D[|p^*(\vec x) - h^*(\vec x))|]\leq \epsilon/2$.
Then, we have
\[
\err_S(f)  \leq  \frac12 \EE_{S}[|y - p(\vec x)|] 
\leq  \frac12 \EE_{S}[|y - p^*(\vec x)|] 
\leq  \frac12 \left( \EE_{S}[|y-h^*(\vec x)|] + \EE_{S}[|p^*(\vec x)-h^*(\vec x)|]\right).
\]
Consider the expected value of the final expression $\frac12( \EE_{S}[|y-h^*(\vec x)|] + \EE_{S}[|p^*(\vec x)-h^*(\vec x)|])$ over the draw of a set $S$ of $m$ samples from $\D$.
Since $|y-h^*(\vec x)|=2$ whenever $y \neq h^*(\vec x)$, we have $\EE_{S \sim \D^m}[\frac12 \EE_S[|y-h^*(\vec x)|]]=\opt$.
Moreover, $\EE_{S \sim \D^m} [\frac12 \EE_{S}[|p^*(\vec x)-h^*(\vec x)|]] \leq \epsilon/4$ since $\EE_{\vec x \sim \D}[|p^*(\vec x)-h^*(\vec x)|]\leq \epsilon/2$.  So, the expected value of the final expression is at most $\opt+\epsilon/4$. 
The expected value of the initial expression $\err_S(f)$ is the expected empirical error of the hypothesis produced. Using the fact that the VC dimension of the class of degree-$\kappa(1/\epsilon)$ polynomial thresholds is $O\left(d^{\kappa(1/\epsilon)}\right)$
and Algorithm~\ref{fig:alg:KKMS} uses $m = \poly\left(\epsilon^{-1} d^{\kappa(1/\epsilon)} \right)$ samples,
the expected empirical error of $h$ is within $\epsilon/4$ of its expected true error, proving the theorem.
\end{proof}

Note that Theorem~\ref{thm:KKMS:exp} bounds the error in expectation over $S\sim \D^m$,  rather than with high probability. This is because  
 $|p^*(\vec x)-h^*(\vec x)|$ is small in expectation but unbounded in the worst-case.
However, this is sufficient to show that a single run of the Algorithm~\ref{fig:alg:KKMS} has $\err_\D(f) \leq \opt+\epsilon$ with probability $\Omega(\epsilon)$. To achieve a high-probability bound, we run this algorithm $O(\frac{1}{\epsilon} \log \frac{1}{\delta})$ times and evaluate the outcomes on a separate sample of size $\tilde O(\frac{1}{\epsilon^2}\log \frac{1}{\delta})$. This is formalized  below.

\begin{corollary}
\label{thm:kkms}
For any $\epsilon$ and $\delta$, 
repeat Algorithm~\ref{fig:alg:KKMS} on $O\left(\epsilon^{-1}\ln(1/\delta)\right)$ independently generated sample sets $S_i$ to learn $f_i$. Take an additional $\tilde{O}\left( \epsilon^{-2} \ln(1/\delta)\right)$ samples from $\D$ and return $f^*_i$ that minimizes the error on this sample set. With probability $1-\delta$, this classifier has $\err_\D(f^*_i)\leq \opt + \epsilon$.
\end{corollary}

An interesting aspect of Algorithm~\ref{fig:alg:KKMS} is that it is \emph{improper}, i.e., it uses a \emph{polynomial threshold} function to learn over the class of linear threshold functions.  
Furthermore, this algorithm runs in polynomial time and learns a classifier of error $O(\mathrm{opt})$ when $\mathrm{opt} = \err_\D(h^*)$ is \emph{an arbitrarily small constant.}
As discussed in Section~\ref{sec:comp-complexity:worstcase}, no computationally efficient algorithm could have obtained this guarantee for general distributions~\citep{daniely2015hardness}. This highlights the need to use structural properties of log-concave distributions for obtaining improved learning guarantees.

While $L_1$ polynomial regression is an extremely powerful algorithm and can obtain error of $\opt+\epsilon$ in $\poly(d^{\kappa(1/\epsilon)})$ time for any value of $\epsilon$, its runtime is polynomial only when $\epsilon$ is a constant.
There is a different simple and efficient algorithm, 
called \emph{Averaging}, that  achieves  non-trivial learning guarantees for any $\epsilon$ that is sufficiently larger than $\opt$.
Averaging~\citep{servedio2001efficient} returns the average of label weighted instances\footnote{Interestingly, Averaging is equivalent to $L_2$ polynomial regression of degree $1$ over the Gaussian distribution.}
\begin{equation*}
\text{classifier }h_{\vec w} \text{ where } \vec w = \frac{\EE_S[y \vec x]}{\|\EE_S[y \vec x]\|}.   \tag{Averaging Algorithm}
\end{equation*}
The idea behind Averaging is simple---if a distribution is realizable and symmetric around the origin (such as a Gaussian) then $\EE[y\vec x]$ points in the direction of the perfect classifier. 
However, 
even a small amount of adversarial noise can create a vector component that is orthogonal to $\vec w^*$.
Nevertheless, \citet{KKMS08} show that Averaging recovers $\vec w^*$ within angle $\epsilon$ when $\opt$ is sufficiently smaller than $\epsilon$.

\begin{theorem}[Based on \citet{KKMS08}]
\label{thm:average}
Consider a distribution $\D$ with a Gaussian and unit variance marginal.
There is a constant $c$ such that for any $\delta$ and $\epsilon > c~\opt \sqrt{\ln(1/\opt)}$,
there is $m\in O\left(\frac{d^2}{\epsilon^2} \ln\big(\frac d\delta\big) \right)$ such that with probability $1-\delta$, the outcome of Averaging on a set of $m$ samples has 
$\err_\D(h_{\vec w}) \leq \opt + \epsilon$.
Furthermore Averaging runs in time $\poly(d, \frac 1\epsilon)$.
\end{theorem}

This theorem shows that when $\epsilon\in \Omega(\opt \sqrt{\ln(1/\opt)})$ one can efficiently learn a linear threshold of error $\opt + \epsilon$. In the next section, we present a  stronger algorithm  that achieves the same learning guarantee for $\epsilon\in \Omega(\opt)$ 
based on an adaptive localization technique that limits the power of adversarial noise further.

\subsection{Computational Improvements via Localization}
\label{sec:ABL}

One of the challenges we face in designing computationally efficient  learning algorithms is that an algorithm's sensitivity to noise is not the same throughout the distribution.
This is often a byproduct of the fact that easy-to-optimize surrogate loss functions, which approximate the non-convex $\sign$ function, have non-uniform approximation guarantee over the space. 
This poses a challenge since an adversary can corrupt a small fraction of the data in more sensitive regions and degrade the quality of the algorithm's outcome disproportionately.
Identifying and removing these regions typically require knowing the target classifier and therefore cannot be fully done as a pre-processing step.
However, when a reasonably good classifier is known in advance it may be possible to approximately identify these regions and \emph{localize} the problem to learn a better classifier. This is the idea behind the work of \citet{ABL17} which creates an adaptive sequence of carefully designed optimization problems based on localization both on the instance space and the hypothesis space,  i.e., focusing on the data close to the current decision boundary and on classifiers close to the current guess. Localization on instances close to the current decision boundary reduces the impact of adversarial noise on the learning process, while localization on the hypothesis space ensures that the history is not forgotten. Using this, \citet{ABL17} learns a linear separator of error $\opt+ \epsilon$ for $\epsilon \in \Omega(\opt)$ in time $\poly(d, \frac 1\epsilon)$ when $\D$ has an isotropic log-concave marginal.

\paragraph{Key Technical Ideas.}
At its core, localization leverages the fact that a large fraction of the disagreement between a reasonably good classifier and the optimal classifier is close to the decision boundary of the former classifier over an isotropic log-concave distribution.
To see this, consider Figure~\ref{fig:margin-based}(a) that demonstrates the disagreement between classifiers $h_{\vec w}$ and $h_{\vec w^*}$ as wedges of  the distribution with total probability,
\begin{equation}
\err_\D(h_{\vec w}) - \err_\D(h_{\vec w^*}) \leq \Pr\Bracks{h_{\vec w}(\vec x) \neq h_{\vec w^*}(\vec x) } \leq \overline{C}_1 \theta(\vec w, \vec w^*),
\label{eq:wedge}
\end{equation}
where $\overline{C}_1$ is a constant according to property~\ref{eq:log:angle} of Theorem~\ref{thm:log:prop}.
This region can be partitioned to instances that are within some $\gamma$-boundary of $\vec w$ (red vertical stripes) and those that are far (blue horizontal stripes).
Since the distribution is log-concave, $\Theta(\gamma)$ fraction of it falls within distance $\gamma$ of the decision boundary of $h_{\vec w}$ (see Theorem~\ref{thm:log:prop} part~\ref{eq:log:band}).
Moreover, instances that are far from 
the decision boundary and are classified differently by $h_{\vec w}$ and $h_{\vec w^*}$ form a small fraction of the total disagreement region.
Formally, using properties \ref{eq:log:band} and \ref{eq:log:wedge} of Theorem~\ref{thm:log:prop}, for chosen constant $C_3 = \underline{C}_1/8$ and 
$\gamma := \alpha \cdot \max \left\{C'_3, \frac{\underline{C}_1}{\overline{C}_2} \right\}$, 
for any $\vec w$ such that $\theta(\vec w, \vec w^*) \leq \alpha$ 
\begin{equation}
\Pr\Bracks{|\vec w\cdot \vec x| \leq \gamma \text{ and }h_{\vec w}(\vec x) \neq h_{\vec w^*}(\vec x) } \leq \Pr\Bracks{|\vec w\cdot \vec x| \leq \gamma} \leq \overline{C}_2 \gamma,  \text{ and }
\label{eq:band-prob}
\end{equation}
\begin{equation}
 \Pr\Bracks{|\vec w\cdot \vec x|> \gamma \text{ and }h_{\vec w}(\vec x) \neq h_{\vec w^*}(\vec x) } \leq  C_3 \alpha \leq \frac{\overline{C}_2}{8} \gamma.
\label{eq:out-prob}
\end{equation}

\begin{figure}
\begin{center}
\epsfig{file=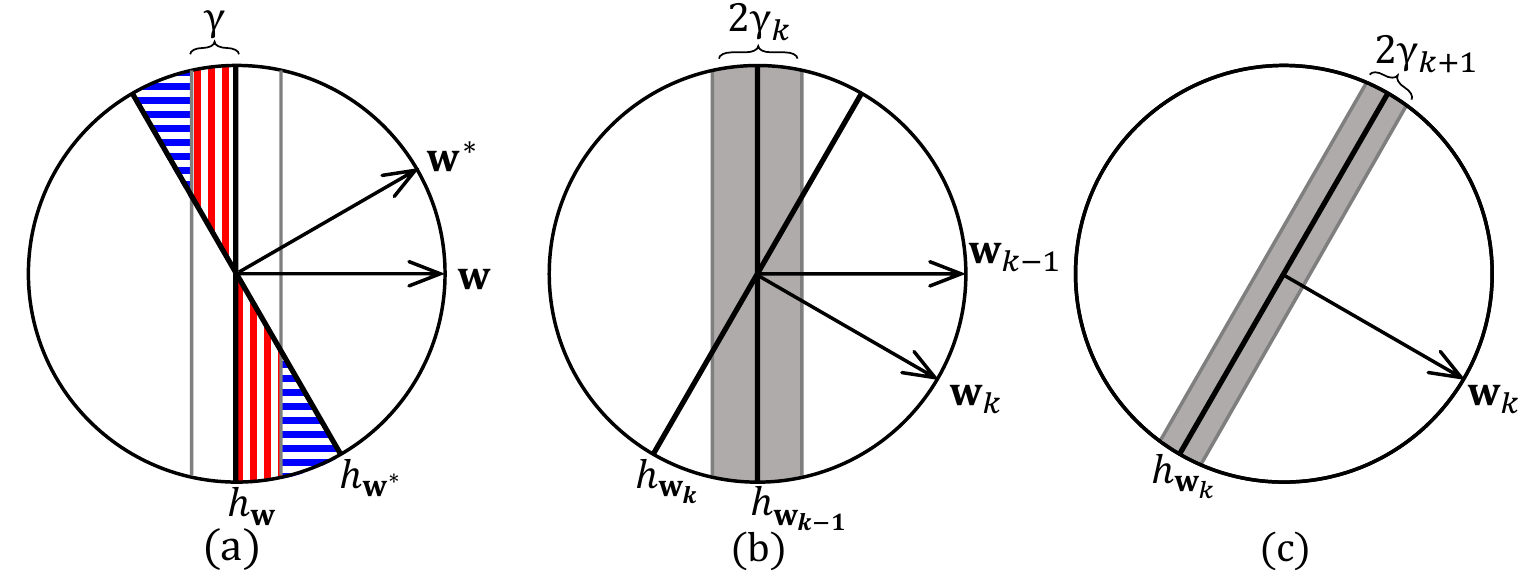,scale= 0.75}
\caption{Demonstrating localization and margin-baesd analysis. Figure (a) demonstrate the partition of disagreement of $h_{\vec w}$ and $h_{\vec w^*}$ to instances within the band and outside the band. Figures (b) and (c) demonstrates the search for $\vec w_{k}$ within $\alpha_k$ angle of $\vec w_{k-1}$ that has error of at most $c_0$ and the band in the subsequent iteration of step (2) of Algorithm~\ref{alg:ABL:oracle}}\label{fig:margin-based}
\end{center}
\end{figure} 

Since $h_{\vec w}$ has low disagreement far from its decision boundary as shown in Equation~\ref{eq:out-prob}, this suggests that to improve the overall performance one can focus on reducing the error near  the decision boundary of $h_{\vec w}$.
Consider $\vec w$, such that $\theta(\vec w, \vec w^*)\leq \alpha$, and any $\vec w'$ within angle $\alpha$ of $\vec w$.
Focusing on instances close to the boundary  $\{(\vec x, y) : |\vec w\cdot \vec x| \leq \gamma \}$, hereafter called the \emph{band}, and the corresponding distribution of labeled instances  in the band denoted by $\D_{\vec w, \gamma}$, if the disagreement of $h_{\vec w'}$ with $h_{\vec w^*}$ in the band is at most $c_0 = \min\left\{ \frac 14, \frac{\underline{C}_1}{4\overline{C}_2 C'_3}  \right\}$
then using the rightmost inequality of Equation~\ref{eq:band-prob} and the leftmost inequality of Equation~\ref{eq:out-prob}, and the fact that 
$\{\vec x : |\vec w\cdot \vec x|> \gamma, h_{\vec w'}(\vec x) \neq h_{\vec w^*}(\vec x)\} \subseteq 
\{ \vec x : |\vec w\cdot \vec x| > \gamma, h_{\vec w}(\vec x) \neq h_{\vec w^*}(\vec x)  \text{ or } h_{\vec w}(\vec x) \neq h_{\vec w'}(\vec x)  \}$, we have

\begin{align}
\Pr_\D\Bracks{h_{\vec w'}(\vec x) \neq h_{\vec w^*}(\vec x)} \leq \overline{C}_2 \gamma \cdot  \Pr_{ \D_{\vec w, \gamma}}\Bracks{h_{\vec w'}(\vec x) \neq h_{\vec w^*}(\vec x)}  + 2\alpha C_3 \leq \frac{\alpha \underline{C}_1 }{2}.
\label{eq:halving}
\end{align}
Thus,  $\theta(\vec w', \vec w^*) \leq \alpha/2$, by property~\ref{eq:log:angle} of Theorem~\ref{thm:log:prop}.
That is, given a classifier $h_{\vec w}$ that is at angle at most $\alpha$ to the optimal classifier, one can find an even better classifier that is at angle at most $\alpha/2$ by searching over classifiers that are within angle $\alpha$ of $\vec w$ and have at most a constant disagreement $c_0$ with the optimal classifier over the band.
This shows that localization reduces the problem of agnostically learning a linear threshold with error $\opt + \epsilon$ on an isotropic log-concave distribution to a learning problem within the band. 

In more detail, for the moment let us assume that we have an oracle for the band that returns $\vec w'$ within angle $\alpha$ of $\vec w^*$ such that the disagreement of $h_{\vec w'}$ with $h_{\vec w^*}$ is at most the previously defined constant $c_0$,  whenever the error of $h_{\vec w^*}$ in the band is small enough compared to $c_0$, i.e., for some fixed function $g(\cdot)$, $\err_{\D_{\vec w,\gamma}}(h_{\vec w^*})\leq g(c_0)$.
\begin{quote}
\underline{Oracle $\mathcal{O}(\vec w, \gamma, \alpha, \delta)$}:
Given $\vec w$, $\gamma$, $\alpha$, $\delta$, and a fixed error tolerance function $g(\cdot)$, such that $\theta(\vec w, \vec w^*)\leq \alpha$ and $\err_{\D_{\vec w, \gamma}}(h_{\vec w^*}) \leq g(c_0)$, the oracle
takes $m(\gamma, \delta, \alpha)$ samples from $\D$ and returns $h_{\vec w'}$ such that $\theta(\vec w', \vec w)\leq \alpha$ and with probability $1-\delta$,
$\Pr_{\D_{\vec w, \gamma}}[ h_{\vec w'}(\vec x) \neq h_{\vec w^*}(\vec x)] \leq c_0.
$
\end{quote}

Algorithm~\ref{alg:ABL:oracle} uses this oracle repeatedly to find a classifier of error $\opt + \epsilon$. We note that as we localize the problem in narrower bands we may increase the overall noise in the conditional distribution of the band. This inherently makes learning more challenging for the oracle.
 Therefore, we use function  $g(\cdot)$ to emphasize the pre-conditions under which the oracle should succeed.

\begin{figure}
\hrule\medskip
\textbf{Input}: Given $\epsilon, \delta$, sample access to $\D$, an in-band optimization oracle $\mathcal{O}$, and an initial classifier $h_{\vec w_1}$ such that $\theta(\vec w_1, \vec w^*) < \pi/2$. 
\begin{enumerate}
\item Let constant $c_{\gamma} = \max\left\{ C'_3, \underline{C}_1 / \overline{C}_2 \right\}$, $\alpha_k = 2^{-k}\pi$ and $\gamma_k =  \alpha_k \cdot c_{\gamma}$, for all $k$.
\item For $k = 1, \dots, \log\big(  \frac{\overline{C}_1\pi}{\epsilon} \big)-1=r$,
let $h_{\vec w_{k+1}} \gets \mathcal{O}\left(\vec w_k, \gamma_k, \alpha_k, \frac{\delta}{r} \right).$
\item Return $h_{\vec w_{r}}$.
\end{enumerate}
\caption{Localization with an Oracle.}\label{alg:ABL:oracle}
\medskip\hrule
\end{figure}

\begin{lemma}[Margin-Based Localization]
\label{lem:oracle-margin}
Assume that oracle $\mathcal{O}$ and a corresponding error tolerance function $g(\cdot)$ exist that satisfy the post-conditions of the oracle on the sequence of inputs $(\vec w_k, \gamma_k, \alpha_k, \delta/r)$ used in Algorithm~\ref{alg:ABL:oracle}.
There is a constant $c$ such that
for any distribution $\D$ with an isotropic log-concave marginal, $\delta$, and $\epsilon \geq c\opt$,
Algorithm~\ref{alg:ABL:oracle}, takes 

\[ m  = \sum_{k = 1}^{\log(\overline{C}_1/\epsilon)} m\left( \gamma_k, \alpha_k, \frac{\delta}{\log(\overline{C}_1/\epsilon)} \right)
\]
samples from $\D$ and returns $h_{\vec w}$ such that with probability $1-\delta$, $\err_\D(h_{\vec w}) \leq \opt + \epsilon$. 
\end{lemma}
\begin{proof}
Algorithm~\ref{alg:ABL:oracle} starts with $\vec w_1$ of angle at most $\alpha_1$ to $\vec w^*$.
Assume for now that $\err_{\D_{\vec w_k, \gamma_k}}(h_{\vec w^*})\leq g(c_0)$ for all $\vec w_k$ and $\gamma_k$  used by the algorithm, so that
the pre-conditions of the oracle are met.
Therefore, at every subsequent Step(2) of the algorithm, the oracle returns
$\vec w_{k+1} \gets \mathcal{O}(\vec w_k, \gamma_k, \alpha_k, \delta/r)$ such that with probability $1-\delta/r$, $\Pr_{\D_{\vec w_k, \gamma_k}}[h_{\vec w_{k+1}}(\vec x) \neq h_{\vec w^*}(\vec x)] \leq c_0$.
Using Equation~\ref{eq:halving}, $\theta(\vec w_{k+1}, \vec w^*) \leq \theta(\vec w_{k}, \vec w^*) /2 \leq\alpha_{k+1}$, i.e., angle of the candidate classifier to $\vec w^*$ is halved.
After $r = \log(\overline{C}_1 \pi/\epsilon)-1$ iterations, we have
$\theta(\vec w_r, \vec w^*) \leq \epsilon/\overline{C}_1$.
Using the relationship between the error of $h_{\vec w}$ and its angle to $h_{\vec w*}$ as described in Equation~\ref{eq:wedge},
$\err_\D(h_{\vec w_r})  \leq  \err_\D(h_{\vec w^*}) + \overline{C}_1\theta(\vec w^*, \vec w_r)  \leq \opt + \epsilon.$ This approach works for all noise types.

Now, we use the properties of the (adversarial) noise to show that the pre-conditions of the oracle are met as long as the width of the band is not much smaller than $\opt$.
That is, for all $k\leq r$, $\err_{\D_{\vec w_k, \gamma_k}}(h_{\vec w^*}) < g(c_0)$ when $\epsilon > c~\opt$.
As we focus on a band of width $\gamma_k$ we may focus on areas where $h_{\vec w^*}$ is wrong. But, any band of width $\gamma_k$
constitutes at least $\underline{C}_2\gamma_k \in \Omega(\epsilon)$
fraction of $\D$ for $k \leq r$.
Hence, there is constant $c$ for which 
$\err_{\D_{\vec w_k, \gamma_k}}(h_{\vec w^*}) \leq \frac{\opt}{\underline{C}_2 \gamma_k} \leq g(c_0),
$ for all $\epsilon \geq c\,\opt$.
\end{proof}

Lemma~\ref{lem:oracle-margin} shows that to get a computationally efficient algorithm for learning over a log-concave distribution, it is sufficient to implement oracle $\mathcal{O}$ efficiently.
We use \emph{hinge loss minimization} for this purpose.
Formally, hinge loss with parameter $\tau$ is defined by
$ \ell_\tau(\vec w, \vec x, y) = \max\Braces{0, 1- \frac{y(\vec w\cdot \vec x)}{\tau}}.$
Note that whenever $h_{\vec w}$ makes a mistake on $(\vec x, y)$ the hinge loss is at least $1$. Therefore,
$\err(h_{\vec w}) \leq \EE\Bracks{\ell_\tau(\vec w, \vec x, y)}.$
Thus, it suffices to show that for any distribution $\D_{\vec w_k, \gamma_k}$ we can find $\vec w_{k+1}$ whose expected $\tau_k$ hinge loss is at most $c_0$. Since hinge loss is a convex function, we can efficiently optimize it over the band using Algorithm~\ref{fig:alg:hinge}. So, the main technical challenge is to show that there is a classifier, namely $h_{\vec w^*}$, whose hinge loss is sufficiently smaller than $c_0$, and therefore, Algorithm~\ref{fig:alg:hinge} returns a classifier whose hinge loss is also less than $c_0$.
This is done via a series of technical steps, where first the hinge loss of $h_{\vec w^*}$ is shown to be small when the distribution has no noise and then it is shown that noise can increase the hinge loss of $h_{\vec w^*}$ by a small amount only.
We refer the reader to the work of \citet{ABL17} for more details.

\begin{figure}[t]
\hrule\medskip
\textbf{Input}: Unit vector $\vec w_k$, $\gamma_k$, $\alpha_k$, $\delta$, and sampling access to $\D$. 
\begin{enumerate}
\item Take a set $S$ of $\tilde\Theta\left(\frac{d^2}{\gamma_k c_0^2} \ln\big( \frac 1 \epsilon \big)\ln\big( \frac 1 \delta\big) \right)$ samples and let $S_k = \Braces{(\vec x, y)   \mid |\vec w_k \cdot \vec x| \leq \gamma_k} $.

\item Let $\tau_k = \gamma_k c_0 \underline{C}_2 / 4 \overline{C}_2$ and for the convex
set $\mathcal{K} = \{ \vec v \mid \| \vec v\| \leq 1 \text{ and } \theta(\vec v, \vec w) \leq \alpha_k\}$ let
$ \vec v_{k+1} \gets \argmin_{\vec v\in \mathcal{K}}\EE_S\Bracks{\ell_{\tau_k}(\vec v, \vec x, y)}. $
\item Return $\vec w_{k+1} = \frac{\vec v_{k+1}}{\|\vec v_{k+1}\|}$.
\end{enumerate}
\caption{Hinge Loss Minimization in the Band.}\label{fig:alg:hinge}
\medskip\hrule
\end{figure}

\begin{lemma}[Hinge Loss Minimization]
\label{lem:hinge}
There is a function $g(z) \in \Theta(z^4)$, such that
for any distribution $\D$ with an isotropic log-concave marginal, 
given $\vec w_k$, $\gamma_k$, and $\alpha_k$ used by Algorithm~\ref{alg:ABL:oracle}, such that $\theta(\vec w_k, \vec w^*)\leq \alpha_k$ and  $\err_{\D_{\vec w_k, \gamma_k}}(h_{\vec w^*})\leq g(c_0)$,
Algorithm~\ref{fig:alg:hinge} takes $n_k = \tilde\Theta\left(\frac{d^2}{\gamma_k c_0^2} \ln(1/\epsilon)\ln(1/\delta) \right)$ samples from $\D$ and returns $\vec w_{k+1}$ such that $\theta(\vec w_{k+1}, \vec w_k)\leq \alpha_k$ and with probability $1-\delta$,
$\Pr_{\D_{\vec w_k, \gamma_k}}[ h_{\vec w_{k+1}}(\vec x) \neq h_{\vec w^*}(\vec x)] \leq c_0.
$
\end{lemma}
Lemmas~\ref{lem:oracle-margin} and \ref{lem:hinge} prove the main result of this section.

\begin{theorem} [\citet{ABL17}]
\label{thm:ABL}
Consider distribution $\D$ with an isotropic log-concave marginal.
There is a constant $c$ such that for all $\delta$ and $\epsilon \geq c~ \opt$, there is $m\in \tilde O\left(\frac{d^2}{\epsilon}\ln\big( \frac 1 \delta \big) \right)$ for which with probability $1-\delta$, Algorithm~\ref{alg:ABL:oracle} using Algorithm~\ref{fig:alg:hinge} for optimization in the band, takes $m$ samples from $\D$,  and returns a classifier $h_{\vec w}$ whose error is $\err_\D(h_{\vec w}) \leq \opt + \epsilon$.
\end{theorem}

\emph{The Surprising Power of Localization.}
The above localization technique can also be used with stronger malicious adversaries who not only can change the label of a fraction of instances, but also \emph{can change the shape of the underlying distribution}. This noise is commonly referred to as \emph{malicious noise} or \emph{poisoning attack}.

Consider applying Algorithm~\ref{alg:ABL:oracle}
in presence of 
malicious noise when the original distribution has an isotropic log-concave marginal. Since malicious noise changes the marginal distribution of  instances, it is not clear if hinge-loss minimization of Lemma~\ref{lem:hinge} can find a suitable classifier $h_{\vec w_{k+1}}$. To deal with this \citet{ABL17} introduce a \emph{soft outlier removal} technique that is applied before  hinge loss minimization in every step of Algorithm~\ref{alg:ABL:oracle}.
At a high level, this procedure assigns  weights to instances in the band, which indicate the algorithm's confidence that these instance were not introduced by ``malicious noise''.
These weights are computed by a linear program that takes into account the variance of instances in the band in directions that are close to $\vec w_k$ in angle. The algorithm uses weighted hinge loss minimization to find  $\vec w_{k+1}$ with similar guarantees to those stated in Lemma~\ref{lem:hinge}. This shows that when the original distribution has an isotropic log-concave marginal, a variant of Algorithm~\ref{alg:ABL:oracle} can deal with  malicious adversaries.

\begin{theorem} [\citet{ABL17}]
\label{thm:ABL-malicious}
Consider a realizable distribution $\D$ with an isotropic log-concave marginal and consider 
a setting where $(1-\opt)$ fraction of the data comes i.i.d. from $\D$ and the other $\opt$ fraction is chosen by a malicious adversary.
There is a constant $c$ such that for all $\delta$ and $\epsilon \geq c~ \opt$, there is an algorithm that takes $m\in \poly(d, \frac 1\epsilon)$ samples and runs in time $\poly(d, \frac 1\epsilon)$ and with probability $1-\delta$ returns a classifier $h_{\vec w}$ whose error is $\err_\D(h_{\vec w}) \leq \epsilon$.
\end{theorem}
The above theorem shows that localization can extend the guarantees of Theorem~\ref{thm:ABL} against stronger adversaries. 
This result  improves over previously known result of \citet{KLS09} which only handle a smaller amount of dimension-dependent noise.

In Section~\ref{sec:22:massart}, we will see that localization is also useful in obtaining much stronger learning guarantees in presence of real-life (and weaker) adversaries that are further constrained in the type of noise they can induce in the data.

\section{Benefits of Assumptions on the Noise}
\label{sec:22:noise}
 
Learning halfspaces can also be studied in a number of intermediate noise settings.
A natural assumption on the noise in classification is that the Bayes optimal classifier belongs to the set of classifiers one considers. That is, any instance is more likely to appear with its correct label rather than the incorrect one, in other words, 
$f_{bayes}(\vec x) = \sign\Parens{\EE[y|\vec x]} = h_{\vec w^*}(\vec x)$ for some $\vec w^*$ in the case of learning linear thresholds.  
This type of noise and its variants are often used to model the noise that is found in  crowdsourced data sets, where the assumption on the noise translates to the belief that any given instance would be correctly labeled by majority of labelers. 
\emph{Random classification noise} with parameter $\nu <\frac 12$ considers a setting where for all $\vec x$, $\EE[y h_{\vec w^*}(\vec x) |\vec x] = (1 - 2 \nu) $. More generally, 
\emph{bounded noise} with parameter $\nu < \frac 12$ only requires that for all $\vec x$, $\EE[y h_{\vec w^*}(\vec x)|\vec x] \geq ( 1 - 2 \nu)$. Equivalently, random classification noise and bounded noise  can be described as noise that is added to a realizable distribution where every instance $\vec x$ is assigned  the wrong label with probability $\nu$ or $\nu(\vec x)\leq \nu$, respectively. Unless stated otherwise, we assume that $\nu$ is bounded away from $\frac 12$ by a constant.

In this section, we explore how assumptions on the niceness of the noise allows us to obtain better computational and statistical learning guarantees.

\subsection{Statistical Improvements for Nicer Noise Models}
\label{sec:22:noise:stat}
A key property of random classification and bounded noise  is that it tightly upper and lower bounds the relationship between the excess error of a classifier and its disagreement with the optimal classifier. That is, for any classifier $h$, 
\begin{equation}
\label{eq:excess-dis}
(1-2\nu) \Pr_{\D}\Bracks{ h(\vec x) \neq h_{\vec w^*}(\vec x)} \leq \err_\D(h) - \err_\D(h_{\vec w*}) \leq  \Pr_{\D}\Bracks{ h(\vec x) \neq h_{\vec w^*}(\vec x)}.
\end{equation}
The right side of this inequality holds by triangle inequality regardless of the noise model. However the left side of this inequality crucially uses the properties of bounded and random classification noise to show that if $h$ and $h_{\vec w^*}$ disagree on $\vec x$, then $\vec x$ and its expected label $\EE[y|\vec x]$ contribute to the error of both classifiers. This  results in  $h$ incurring only a small excess error over the error of $h_{\vec w^*}$.

Equation~\ref{eq:excess-dis} is particularly useful because its right-hand side, which denotes the disagreement between $h$ and $h_{\vec w^*}$,  is also
the variance of $h$'s excess error, i.e., 
\[ \EE_{\D}\Bracks{ \Parens{\err_{(\vec x,y)}(h) - \err_{(\vec x,y)}(h_{\vec w^*})}^2} = \Pr_{\D}\Bracks{ h(\vec x) \neq h_{\vec w^*}(\vec x)}.
\]
Therefore, an upper bound on the disagreement of $h$ also bounds the variance of its excess error and allows for stronger concentration bounds. For example, using Bernstein's inequality and the VC theory, with probability $1-\delta$ over the choice of a set $S$ of $m$ i.i.d. samples from $\D$, for all linear thresholds $h$ we have
\[ \err_\D(h) - \err_\D(h_{\vec w^*}) \leq \err_S(h) - \err_S(h_{\vec w^*}) + \sqrt{\frac{ \Parens{\err_\D(h) - \err_\D(h_{\vec w^*})} (d + \ln(\frac 1 \delta)) }{(1-2\nu) m}} + O\Parens{\frac {1}{m}}.
\]
That is, there is $m \in O\left(\frac{d + \ln(1/ \delta)}{(1-2\nu)\epsilon}  \right)$ such that with probability $1-\delta$ the classifier $h'$ that minimizes the empirical error on $m$ samples has  $\err_\D(h) \leq \opt+ \epsilon$.

This shows that if $\D$ demonstrates bounded or random classification noise it can be learned with fewer samples than needed in the agnostic case. 
This is due to the fact that we directly compare the error of $h$ and that of $h_{\vec w^*}$ (and using stronger concentration bounds) instead of going through uniform convergence.
While we need $\Omega(d/\epsilon^2)$ samples to learn a classifier of error $\opt + \epsilon$  in the agnostic case, we only need $O(d/\epsilon)$ samples to learn in presence of random classification or bounded noise.
We note that this result is purely information theoretical, i.e., it does not imply existence of a polynomial time  algorithm that can learn in presence of this type of noise. In the next section, we discuss the issue of computational efficiency in presence of random classification and bounded noise in detail.

\subsection{Further Computational Improvements for Nicer Noise Models}
\label{sec:22:massart}
In this section, we show that in presence of random classification or  bounded noise there are computationally efficient algorithms that enjoy improved noise robustness guarantees compared to their agnostic counterparts.
In particular, one can efficiently learn linear thresholds in presence of random classification noise.
\begin{theorem}[\citet{BFKV98}]
For any distribution $\D$ that has random classification noise, there is an algorithm that runs in time $\poly(d, \frac 1\epsilon, \ln(1/\delta))$ and with probability $1-\delta$ learns a vector $\vec w$ such that $\err_\D(h_{\vec w}) \leq \opt + \epsilon$.  
\end{theorem}
We refer the interested reader to the work of \citet{BFKV98} for more details
on the algorithm that achieves the above guarantees.
Let us note that when in addition to random classification noise, which is a highly symmetric noise, the marginal distribution is also symmetric several simple algorithms can learn a classifier of error $\opt + \epsilon$. For example, when the distribution is Gaussian and has random classification noise the Averaging algorithm of Section~\ref{sec:22:poly-reg} recover $\vec w^* \propto \EE_\D[\vec x y]$.

While the random classification noise leads to a polynomial time learning algorithm,
it does not present a convincing model of learning beyond the worst-case. In particular, the highly symmetric nature of random classification noise does not lend itself to real-world settings where parts of data may be less noisy than others.
This is where bounded noise comes into play---it relaxes the symmetric nature of random classification noise, yet, assumes that no instance is too noisy.
As opposed to the random classification noise, however, no efficient algorithms are known to date that can learn a classifier of error $\opt + \epsilon$ in presence of bounded noise when the marginal distribution is unrestricted.
Therefore, in the remainder of this section we focus on a setting where, in addition to having bounded noise, $\D$ also has a nice marginal distribution, specifically an isotropic log-concave  marginal distribution.

Let us first consider the
iterative localization technique of Section~\ref{sec:ABL}.
Given that bounded noise is a stronger assumption than adversarial noise, Theorem~\ref{thm:ABL} implies that for small enough $\opt \in O(\epsilon)$, we can learn a linear threshold of error $\opt + \epsilon$.
Interestingly, the same algorithm achieves a much better guarantee for bounded noise and the key to this is that 
$\err_{\D_{\vec w, \gamma}}(h_{\vec w^*}) \leq \nu$ for any $\vec w$ and $\gamma$.

\begin{lemma}[Margin-Based Localization for Bounded Noise]
\label{lem:oracle-margin-bounded}
Assume that oracle $\mathcal{O}$ and a corresponding error tolerance function $g(\cdot)$ exist that satisfy the post-conditions of the oracle on the sequence of inputs $(\vec w_k, \gamma_k, \alpha_k, \delta/r)$ used in Algorithm~\ref{alg:ABL:oracle}.
For any distribution $\D$ with an isotropic log-concave marginal and $\nu$-bounded noise such that $\nu\leq g(c_0)$ and any $\epsilon$, $\delta$,
Algorithm~\ref{alg:ABL:oracle} takes
\[ m  = \sum_{k = 1}^{\log(\overline{C}_1/\epsilon)} m\left( \gamma_k, \alpha_k, \frac{\delta}{\log(\overline{C}_1/\epsilon)} \right)
\]
samples from $\D$, and returns $h_{\vec w}$  such that with probability $1-\delta$, $\err_\D(h_{\vec w}) \leq \opt + \epsilon$. 
\end{lemma}

The proof follows that of Lemma~\ref{lem:oracle-margin}, with the exception that the noise in the band never increases beyond $\nu \leq g(c_0)$. This is due to the fact that the probability that an instance $\vec x$ is noisy in any band of $\D$ is at most $\nu\leq g(c_0)$.
Note that the preconditions of the oracle are met for arbitrarily small bands when the noise is bounded, as opposed the adversarial setting where the preconditions are met only when the width of the band is larger than $\Omega(\opt)$. So, by using hinge loss minimization in the band (Lemma~\ref{lem:hinge}) one can learn a linear threshold of error $\opt + \epsilon$ when the noise parameter $\nu < g(c_0)$ is a \emph{small constant.} This is much better than our adversarial noise guarantee where the noise has to be at most $\opt < \epsilon/c$.

\begin{figure}[t]
\begin{center}
\epsfig{file=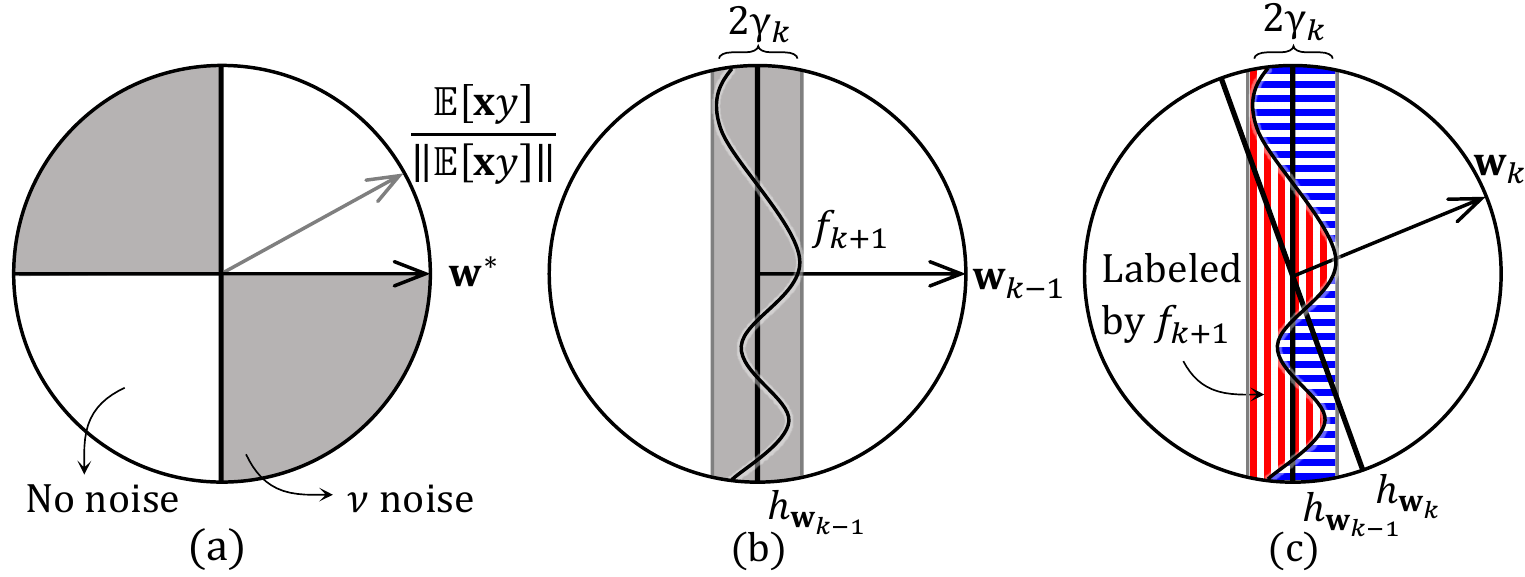,scale= 0.75}
\caption{Figure (a) demonstrates that Averaging performs poorly on bounded noise, even when the distribution is symmetric. Figure (b) demonstrates Step~\ref{step:poly-massart} of Algorithm~\ref{fig:alg:hinge/poly} where polynomial regression is used to learn $f_{k+1}$. Figure (c) demonstrates the use of hinge loss minimization in Step~\ref{step:hinge-massart} of Algorithm~\ref{fig:alg:hinge/poly} on the distribution labeled by $f_{k+1}$, where horizontal and vertical stipes denote regions labeled as positive and negative by $f_{k+1}$.}\label{fig:massart}
\end{center}
\end{figure}

How small $g(c_0)$ is in Lemma~\ref{lem:hinge} and how small is $\nu$ as a result?
As \citet{ABHU15} showed, $\nu$ has to be almost negligible, of the order of $10^{-6}$. So, we ask whether an alternative algorithm can handle bounded noise for \emph{any} $\nu \leq \frac 12 - \Theta(1)$.
Note that the key property needed for applying hinge loss minimization in the band is that the noise in $\D_{\vec w_k, \gamma_k}$ is at most $g(c_0)$, regardless of the value of $\nu$. 
So a natural approach for guaranteeing this property is to de-noise the data in the band and reduce the noise from an arbitrary constant $\nu$ to a smaller constant $g(c_0)$. 
Polynomial regression (Algorithm~\ref{fig:alg:KKMS}) can be used here, as it can learn a polynomial threshold $f_{k+1}$ which has a small constant error in polynomial time.
Had $f_{k+1}$ been a linear threshold, we would have set $h_{\vec w_{k+1}} = f_{k+1}$ and continued with the next round of localized optimization.
However, for general polynomial thresholds we need to approximate $f_{k+1}$ with a linear threshold.
Fortunately, $f_{k+1}$ is already close to  $h_{\vec w^*}$. Therefore, the hinge loss minimization technique of Algorithm~\ref{fig:alg:hinge} can be used to learn a linear threshold $h_{\vec w_{k+1}}$ whose predictions are close to $f_{k+1}$, and thus, is close to $h_{\vec w^*}$. This is formalized in Algorithm~\ref{fig:alg:hinge/poly} and the following lemma.

\begin{figure}
\hrule\medskip
\textbf{Input}: Unit vector $\vec w_k$, $\gamma_k$, $\alpha_k$, $\delta$, $c_0$, and sampling access to $\D$. 
\begin{enumerate}
\item Take $n_k = \poly\Parens{d^{\poly\big(\frac{1}{1 - 2\nu} \big)}, \frac 1\epsilon, \ln\big(\frac 1 \delta\big)}$ i.i.d. samples from $\D$ and let $S_k$ include the samples for which $\Braces{\vec x \mid |\vec w_k \cdot \vec x| \leq \gamma_k}$. Let $f_{k+1}$ be the outcome of Algorithm~\ref{fig:alg:KKMS} with excess error of $(1-2\nu)g(c_0)$.
\label{step:poly-massart}
\item Take $\tilde\Theta\left(\frac{d^2}{\gamma_k c_0^2} \ln\big( \frac 1 \epsilon \big)\ln\big( \frac 1 \delta\big) \right)$  i.i.d samples from $\D$ and let $S'_k$ include these samples $(\vec x, f_{k+1}(\vec x))$ for which $\Braces{\vec x \mid |\vec w_k \cdot \vec x| \leq \gamma_k}$. 
Let $\tau_k = \frac{c_0 \underline{C}_2}{4 \overline{C}_2} \gamma_k$ and for the convex
set $\mathcal{K} = \{ \vec v \mid \| \vec v\| \leq 1 \text{ and } \theta(\vec v, \vec w) \leq \alpha_k\}$ let
$\vec v_{k+1} \gets \argmin_{\vec v\in \mathcal{K}}\EE_{S'_k}\Bracks{\ell_{\tau_k}(\vec v, \vec x, y)}. $
\label{step:hinge-massart} 
\item Return $\vec w_{k+1} = \frac{\vec v_{k+1}}{\|\vec v_{k+1}\|}$.
\end{enumerate}
\caption{Polynomial Regression and Hinge Loss Minimization in the Band.}\label{fig:alg:hinge/poly}
\medskip\hrule
\end{figure}

\begin{lemma}[Polynomial Regression with Hinge Loss Minimization]
\label{lem:hinge-poly}
Consider distribution $\D$ with an isotropic log-concave marginal and bounded noise with parameter $\nu$.
For any $\epsilon$, $\delta$, $\vec w_k$, $\gamma_k$, and $\alpha_k$ stated in Algorithm~\ref{alg:ABL:oracle},  such that $\theta(\vec w_k, \vec w^*)\leq \alpha_k$, Algorithm~\ref{fig:alg:hinge/poly} takes $n_k = \poly\Parens{d^{\poly\big(\frac{1}{1 - 2\nu} \big)}, \frac 1\epsilon, \ln\big(\frac 1 \delta\big)}$ samples from $\D$ and returns $\vec w_{k+1}$ such that $\theta(\vec w_{k+1}, \vec w_k)\leq \alpha_k$ and with probability $1-\delta$,
$\Pr_{\D_{\vec w_k, \gamma_k}}[ h_{\vec w_{k+1}}(\vec x) \neq h_{\vec w^*}(\vec x)] \leq c_0.
$
\end{lemma}
\begin{proof}
Let $g(c_0) \in \Theta(c_0^4)$ be the 
error tolerance function of hinge loss minimization according to Lemma~\ref{lem:hinge}.
Note that for log-concave distribution $\D$, the distribution in any band is also log-concave. Therefore,  Step~\ref{step:poly-massart} of Algorithm~\ref{fig:alg:hinge/poly} uses polynomial regression to
 learn a polynomial threshold $f_{k+1}$ such that $\err_{\D_{\vec w_k, \gamma_k}}(f_{k+1}) \leq \opt + (1-2\nu) g(c_0)$. 
 Now let distribution $\P$ be the same as $\D$, except that all instances are labeled according to $f_{k+1}$. 
 Since the noise is bounded, by Equation~\ref{eq:excess-dis}
\[ \err_{\P_{\vec w_k, \gamma_k}}\!\!\!(h_{\vec w^*})=  \!\!\!\Pr_{\D_{\vec w_k, \gamma_k}}\!\!\!\Bracks{ f_{k+1}(\vec x)  \neq h_{\vec w^*}(\vec x) } 
\leq \frac{1}{1-2\nu} \Parens{ \!\err_{\D_{\vec w_k, \gamma_k}}\!\!\!(f_{k+1})  - \!\err_{\D_{\vec w_k, \gamma_k}}\!\!\!(h_{\vec w^*})  } 
\leq g(c_0).
\]
Then, distribution $\P$ meets the conditions of   Lemma~\ref{lem:hinge}.
Therefore, Algorithm~\ref{fig:alg:hinge/poly}
returns a $h_{\vec w_{k+1}}$ such that 
$$
\Pr_{\D_{\vec w_k, \gamma_k}}\Bracks{ h_{\vec w_{k+1}}(\vec x) \neq h_{\vec w^*}(\vec x)  }
=
\Pr_{\P_{\vec w_k, \gamma_k}}\Bracks{ h_{\vec w_{k+1}}(\vec x) \neq h_{\vec w^*}(\vec x)  }
\leq c_0.$$
This completes the proof of the lemma.
\end{proof}

Using margin-based localization iteratively while applying polynomial regression paired with hinge loss minimization in the band proves the following theorem.

\begin{theorem}[\citet{awasthi2016learning}]
\label{thm:massart-naive}
Let $\D$ be a distribution with an isotropic log-concave marginal and bounded noise with parameter $\nu$. 
For any $\epsilon$ and $\delta$ there is 
$m = \poly\Parens{d^{\poly\big(\frac{1}{1 - 2\nu} \big)}, \frac 1\epsilon, \ln\big(\frac 1 \delta\big)}$
such that  Algorithm~\ref{alg:ABL:oracle} that uses Algorithm~\ref{fig:alg:hinge/poly} for optimization in the band takes $m$ samples from $\D$, 
runs in time $\poly(m)$, and with probability $1-\delta$ returns a classifier $h_{\vec w}$ whose error is $\err_\D(h_{\vec w}) \leq \opt + \epsilon$.
\end{theorem}

The above theorem shows that as long as $\nu \leq \frac 12 - \Theta(1)$, there is a polynomial time algorithm that learns a linear threshold of error $\opt + \epsilon$ over isotropic log-concave distributions. However, the sample complexity and runtime of this algorithm is inversely exponential in $1-2\nu$. 
Note that this sample complexity is exponentially larger than the  information theoretic bounds presented in Section~\ref{sec:22:noise:stat}. It remains to be seen if there are computationally efficient algorithms that match the information theoretic bound for general log-concave distributions.

\section{Final Remarks and Current Research Directions}

Connecting this chapter to the broader vision of machine learning, let us note that machine learning's  effectiveness  in today's world was directly  influenced by early works on foundational aspects that moved past the worst-case and leveraged properties of real-life learning problems, e.g., finite VC dimension and margins~\citep{CS00}.
We next  highlight some of the current research directions in connection to the  beyond the worst case analysis of algorithms.

\emph{Adversarial Noise.}
The polynomial regression of Section~\ref{sec:22:poly-reg} is due to \citet{KKMS08}. The Averaging results of Theorem~\ref{thm:average} are a variant of \citet{KKMS08} original results that applied to the uniform distribution.
 \citet{KLS09} showed that a variant of Averaging algorithm with a hard outlier removal technique achieves error of $\opt + \epsilon$ for $\opt \in O(\epsilon^3/\ln(1/\epsilon))$ when the distribution is isotropic log-concave. The margin-based localization technique of Lemma~\ref{lem:oracle-margin} and its variants first appeared in \citet{Balcan07} in the context of \emph{active learning}.
The combination of  margin-based localization technique and hinge loss minimization of Section~\ref{sec:ABL} is due to \citet{ABL17} that also works in the active learning setting. \citet{daniely2015ptas} used this technique  with polynomial regression to obtain a PTAS for learning linear thresholds over the uniform distribution. \citet{diakonikolas2018learning} further extended the margin-based localization results to non-homogeneous linear thresholds. Going forward, generalizing these techniques to more expressive hypothesis classes is an important direction for future work. 
The key challenge here is to define an appropriate localization area. For linear separators, we derived margin-based localization analytically. In more general settings, however, such as deep neural networks, a closed form derivation may not be possible. It would be interesting to see if one can instead algorithmically compute a good localization area using the properties of the problem at hand, e.g., by using  unlabeled data.

\emph{Bounded Noise.}
The results of Section~\ref{sec:22:massart} for log-concave distributions with bounded noise are due to \citet{ABHU15} and \citet{awasthi2016learning}.
\citet{yan2017revisiting} used a variant of this algorithm with improved sample complexity and runtime dependence on $1/(1-2\nu)$ for the special case of uniform distribution over the unit ball.
Recently, \citet{diakonikolas2019distribution} presented a polynomial time algorithm for distribution-independent error bound of $\nu + \epsilon$ when noise is $\nu$-bounded and also showed that their techniques and variants thereof fall short of learning a classier of error $\opt + \epsilon$. 
This is a significantly weaker  guarantee because in typical applications $\opt$ is much smaller than $\nu$, which indicates the maximum amount of noise on a given point.
To date, the question of whether there are computationally efficient algorithms or hardness results for getting distribution-independent error of $\opt+\epsilon$ in presence of bounded noise remains an important open problem in the theory of machine learning.
On the other hand, one of the main motivations of bounded noise and its variants is crowdsourcing, where every instance is  correctly labeled by at least $1-\nu$ fraction of labelers. If one designs the data collection protocol as well as  the learning algorithm, \citet{awasthi2017efficient} showed that any set of classifiers $\F$ that can be efficiently learned in the realizable setting using $m_{\epsilon, \delta}^{real}$ samples can be learned efficiently by making $O(m_{\epsilon, \delta}^{real})$ queries to the crowd. This effectively shows that the computational and statistical aspects of \emph{non-persistent bounded noise} are the same as those of the realizable setting.

In addition to the real-life motivation for this noise model, bounded noise is also related  to other notions of beyond the worst-case analysis of algorithms. For example, Equation~\ref{eq:excess-dis}, which  relates the excess error of a classifier to how its predictions differ from that of the optimal classifier, is a supervised analog of the \emph{``approximation stability''} assumption used in \emph{clustering}, that states that any clustering that is close in objective value to the optimal classifier should also be close to it in classification.

\emph{Robustness to other Adversarial Attacks.}
As mentioned, the localization technique introduced in this chapter can also handle malicious noise~\citep{ABL17}.
A related model considers \emph{poisoning attacks} where an adversary inserts maliciously crafted fake data points into a training set in order to cause specific failures to a learning algorithm.  It would be interesting to provide additional formal guarantees for such adversaries.
Another type of attack, called \emph{adversarial examples},  considers a type of corruption that only affects the distribution at test time, thereby requiring one to learn a classifier $f\in \F$ on distribution $\D$ that still achieves a good performance when $\D$ is corrupted by some noise~\citep{goodfellow2014explaining}. By in large, this learning model draws motivation from audio-visual attacks on learning systems, where the goal is to secure learning algorithms against an adversary who is intent on causing harm through misclassification~\citep{KurakinGB16}.
A beyond the worst-case perspective on test-time robustness could also improve the robustness of learning algorithms to several non-adversarial corruptions, such as distribution shift and misspecification, and therefore is a promising direction for future research.

\bibliographystyle{plainnat}
\bibliography{../../../Nika_bib}

\section*{Exercises}
\begin{enumerate}
\item Prove the properties of log-concave distributions, as described in Theorem~\ref{thm:log:prop}, for $\D$ that is a Gaussian distributions with unit variance.

\item Prove that distribution $\D$ has $\nu$-bounded noise if and only if $\EE[y h_{\vec w^*}(\vec x)|\vec x] \geq ( 1 - 2 \nu)$
 for all $\vec x$ in the instance space (except for a measure zero subset). Similarly, prove that $\D$ has  random classification noise with parameter $\nu$ if an only if $\EE[y h_{\vec w^*}(\vec x) |\vec x] = (1 - 2 \nu) $ for all $\vec x$.

\item \label{exer:excess-diff}
For a distribution with $\nu$-bounded noise prove Equation~\ref{eq:excess-dis}.

\end{enumerate}

\end{document}